\documentclass[conference]{IEEEtran}
\IEEEoverridecommandlockouts
\usepackage{cite}
\usepackage{amsmath,amssymb,amsfonts}
\usepackage{algorithmic}
\usepackage{graphicx}
\usepackage{textcomp}
\usepackage{xcolor}
\def\BibTeX{{\rm B\kern-.05em{\sc i\kern-.025em b}\kern-.08em
    T\kern-.1667em\lower.7ex\hbox{E}\kern-.125emX}}
    
\usepackage[square,numbers]{natbib}
\usepackage[colorlinks=false]{hyperref}

\usepackage{booktabs}

\usepackage{amsmath, amsthm, amssymb, amsfonts, physics}

\newtheorem{prop}{Proposition}[section]
\newtheorem{lem}{Lemma}[section]
\theoremstyle{definition}

\theoremstyle{definition}

\DeclareMathOperator*{\argmax}{arg\,max}

\newtheorem{theorem}{Theorem}[section]

\begin{document}

\title{Random Projections for\\Improved Adversarial Robustness}

\author{
\IEEEauthorblockN{1\textsuperscript{st} Ginevra Carbone}
\IEEEauthorblockA{\textit{Dept. of Mathematics and Geosciences} \\
\textit{University of Trieste }\\
Trieste, Italy\\
ginevra.carbone@phd.units.it}\\
\and
\IEEEauthorblockN{2\textsuperscript{rd} Guido Sanguinetti}
\IEEEauthorblockA{\textit{School of Informatics} \\
\textit{University of Edinburgh}\\
Edinburgh, United Kingdom \\
\textit{SISSA}\\
Trieste, Italy \\
gsanguin@sissa.it}
\and
\IEEEauthorblockN{3\textsuperscript{nd} Luca Bortolussi}
\IEEEauthorblockA{\textit{Dept. of Mathematics and Geosciences} \\
\textit{University of Trieste}\\
Trieste, Italy \\
\textit{Modeling and Simulation Group} \\
\textit{Saarland University}\\
Saarland, Germany \\
luca.bortolussi@gmail.comne}
}
\newcommand{\lb}[1]{{\color{red}[Luca: #1]}}
\newcommand{\gs}[1]{\color{red}[Guido: #1]}
\newcommand{\gc}[1]{{\color{blue}[Gin: #1]}}

\maketitle

\begin{abstract}

We propose two training techniques for improving the robustness of Neural Networks to adversarial attacks, i.e. manipulations of the inputs that are maliciously crafted to fool networks into incorrect predictions. Both methods are independent of the chosen attack and leverage random projections of the original inputs, with the purpose of exploiting both dimensionality reduction and some characteristic geometrical properties of adversarial perturbations. The first technique is called \textit{RP-Ensemble} and consists of an ensemble of networks trained on multiple projected versions of the original inputs. The second one, named \textit{RP-Regularizer}, adds instead a regularization term to the training objective.
\end{abstract}


\section{Introduction}

Adversarial examples  \cite{szegedy2014intriguing} are small perturbations of the input data, specifically designed to induce wrong predictions in machine learning models, even for those achieving exceptional accuracy and with a high confidence in the wrong predictions. Such perturbations are often not even recognizable by humans \cite{kurakin2017physical, carlini2019onevaluating}, thus developing suitable defense strategies is crucial in security-critical settings, and especially in computer vision algorithms (e.g. road signs recognition, medical imaging or autonomous driving~\cite{eykholt17robust}).

Defense research is currently focusing on different strategies for preventing this kind of vulnerability: deriving exact robustness bounds under some theoretical constraints~\cite{fawzi2018analysis}, analyzing the robustness to particularly strong attacks~\cite{carlini2019onevaluating}, designing defences that are specific for the chosen attack~\cite{goodfellow2015explaining}, or developing general training algorithms and regularization techniques which improve resilience to multiple attacks.

Random projections of the input samples into lower-dimensional spaces have been extensively used for dimensionality reduction purposes, but in this work we are mostly interested in using them for providing robustness and regularization guarantees against the adversaries. Our main inspiration for this work is the \emph{Manifold Hypothesis}~\cite{genovese2012manifold, dasgupta2008random, donoho2003hessian}, which models data as being sampled from low-dimensional manifolds, corresponding to the classification regions, embedded in a high-dimensional space \cite{khuory18geometry}. Therefore, decision boundaries are represented as hypersurfaces of the embedding space. Such approach allows to face the problem of high-dimensionality of the input space, since the number of samples required for learning grows exponentially with the dimension of the space. Geometrical inspections related to this phenomenon lead us to the idea of using random projections of the input data as a defense. We observe that projected versions of the original data are easier to learn and lie in less complex regions of the space. 

We propose a training technique, called \emph{RP-Ensemble} (\ref{sec:rp_ensemble}), which improves the robustness to adversarial examples of a pre-trained classifier. This method projects the input data in multiple lower dimensional spaces, each one determined by a random selection of directions in the space. Then, it trains a new classifier in each subspace, using the corresponding projected version of the data. Finally, it performs an ensemble classification on the original high dimensional data. In Sec. \ref{sec:rp_regularizer} we also define a regularization term for the training objective, named \emph{RP-Regularizer}. 
This technique combines the norm of the loss gradients, intended as a measure of vulnerability, and the expectation over random projections of the inputs. In doing so, we aim at exploiting relevant adversarial features during training.

We evaluate the adversarial vulnerability of the resulting trained models and compare them to adversarially trained robust models (Sec. \ref{sec:experiments}).
Finally, we discuss the 
scalability and parallelizability of RP-Ensemble. 

\section{Methodology}

In the next sections we will refer to $d$-dimensional data samples and to neural network models of type $f(\cdot,\theta):\mathbb{R}^d\rightarrow\mathbb{R}^K$, with learnable weights $\theta$, solving a classification problem on $K$ classes. 

\subsection{RP-Ensemble}
\label{sec:rp_ensemble}



RP-Ensemble method is built upon a pre-trained model and can be regarded as a fine tuning technique for adversarial robustness.  Let $X\in\mathbb{R}^{n\times d}$ be the $n$ original $d$-dimensional training examples, represented in matrix form, and let $g(\cdot,\theta):\mathbb{R}^d\rightarrow\mathbb{R}^K$ be the pre-trained network. 

First, we project the whole dataset into $p$ different subspaces of dimension $k\leq d$, using Gaussian random projection matrices \cite{johnson-lind}. Each projection matrix $R_j\in$ $\mathbb{R}^{k\times d}$ maps the input data $X$ into its $k$-dimensional projected version $\mathcal{P}_j(X) = X R_j^T \in\mathbb{R}^{n\times k}$, using $k$ random directions.

The elements of each random matrix $R_j$ are independently drawn from a $\mathcal{N}(0,1/k)$ distribution. This particular choice is motivated by Johnson-Lindenstrauss Lemma (\ref{joh_lin_lemma}), ensuring that the Euclidean distance between any two points in the new low-dimensional space is approximately very close to the distance between the same points in the original high-dimensional space \cite{johnson-lind}.

\begin{lem} [Johnson-Lindenstrauss Lemma]
Given a set of n points $\mathcal{M}\subset\mathbb{R}^d$, let $\epsilon\in(0,1/2)$, $k>8\log n/\epsilon^2$ and $A\in\mathbb{R}^{k\times d}$ be a matrix whose entries have been sampled independently from  $\mathcal{N}(0,1/k)$. 

Then, for any couple of points $u,v\in\mathcal{M}$ the following inequality holds
\begin{align*}
P\Big[(1-\epsilon)||u-v||^2
&\leq\Big|\Big|Au-Av\Big|\Big|^2\\
&\leq(1+\epsilon)||u-v||^2\Big]\geq 1-2e^{-(\epsilon^2-\epsilon^3)k/4}.
\end{align*}
\label{joh_lin_lemma}
\end{lem}

Notice that any two independently randomly chosen vectors in a high dimensional space are almost orthogonal with probability close to one and nearly have the same length \cite{kaski98}. Consequently, for any given number of sample points $n$, the $k$-dimensional columns of a random matrix $R_j$ generated from a Gaussian distribution are almost orthogonal to each other. This procedure yields to a projection in the subspace generated by the columns of $R_j$.

Next, we train a classifier $\psi_j(\cdot,\theta_j):\mathbb{R}^{k}\rightarrow\mathbb{R}^K$ in each projected subspace on the corresponding projected version of the data $\mathcal{P}_j(X)$. The architecture of the $\psi_j$-s mirrors that of the pre-trained network $g$, except for the first layer, which is adapted to the size of the projected lower dimensional input. Notice that the $\psi_j$-s do not share their weights nor the inputs, thus the backpropagation algorithm during the training phase stops at the projected data $\mathcal{P}_j(X)$.

Finally, we perform an ensemble classification on the original high dimensional data, by summing up the probability distributions from all the projected classifiers together with the predictions from the pre-trained classifier $g$.
Let $p_{\theta_j}, p_g$ be the probability mass functions for the classifiers $\psi_j$ and $g$. The classification of an input sample $x_i\in\mathcal{X}$ is given by

$$ 
y_i:=\argmax_{y=1,\ldots,K}\Bigg(\sum_{j=1,\ldots,p} p_{\theta_j}(y|\mathcal{P}_j(x_i)) + p_g(y|x_i)\Bigg),
$$

for each $i=1,\ldots,n$.

\subsection{RP-Regularizer}
\label{sec:rp_regularizer}

RP-Regularizer is a variant of Total Variation regularization, a well known denoising approach in image processing \cite{rudin1992totalvariation}, already used in \cite{finlay19improved} as a regularization term for improving adversarial robustness. Its computation for a network $f(\cdot,\theta)$ is made tractable by a numerical approximation on the labeled dataset $\mathcal{D}=\{(x_i,y_i)\}_{i=1,\ldots,n}$, e.g.
$$
\big|\big|\nabla_x f\big|\big|_{L_1}\approx \frac{1}{n}\sum_{i=1}^n \big|\nabla_x \ell(f(x_i, \theta),y_i)\big|,
$$
which is less complex w.r.t. the computation of the full gradient $\nabla_x f$.
Here $\ell$ is the training loss function and $n$ is the number of training samples.
Our regularization term, instead, is computed in the $L_2$-norm on suitable random projections of the input points.

The first step consists in sampling the components of the random matrices $R_j\in\mathbb{R}^{k_j\times d}$ from a Gaussian distribution, as was done in section \ref{sec:rp_ensemble} for RP-Ensemble, but with a randomly chosen projection size $k_j=1,\ldots,d$. Then, we project the input data matrix $X$ in a $k$-dimensional subspace
$$
\mathcal{P}_j(X)= X R_j^T \in \mathbb{R}^{n\times k_j},
$$
$R_j$ being the  $j$-th projection matrix, for all the projection indexes $j=1,\ldots,p$.

Since the penalty for the objective needs to depend on the network's weights at the current training step, we want to map the projections $\mathcal{P}_j(X)$ back into the original $d$-dimensional space. We do this by means of \emph{Moore-Penrose pseudo-inverse}~\cite{penrose_1955} $R^\dagger_j$ of $R_j$ and apply it to the projected points $$\mathcal{P}^\dagger_j(\mathcal{P}_j(X))= X R_j^T (R_j^\dagger)^T \in \mathbb{R}^{n\times d}.$$

Fig. \ref{fig:invproj} shows an example of this procedure on the MNIST dataset~\cite{mnist}. In a nutshell, it builds on the two projection operators 
\begin{align*}
	\mathcal{P}: \,\mathbb{R}^{n\times d}\longrightarrow \prod_{j=1}^p\mathbb{R}^{n\times k_j}\\
	\mathcal{P}^\dagger: \prod_{j=1}^p\mathbb{R}^{n\times k_j}\longrightarrow \prod_{j=1}^p\mathbb{R}^{n\times d}.
\end{align*}

The pseudo-inverse is a generalized inverse matrix. It exists and is unique for any given real rectangular matrix and the resulting composition $R_j^T (R_j^\dagger)^T$ is an orthogonal projection operator on $\mathbb{R}^d$.

\begin{figure}[!htbp]
	\center
	\includegraphics[width=0.85\linewidth]{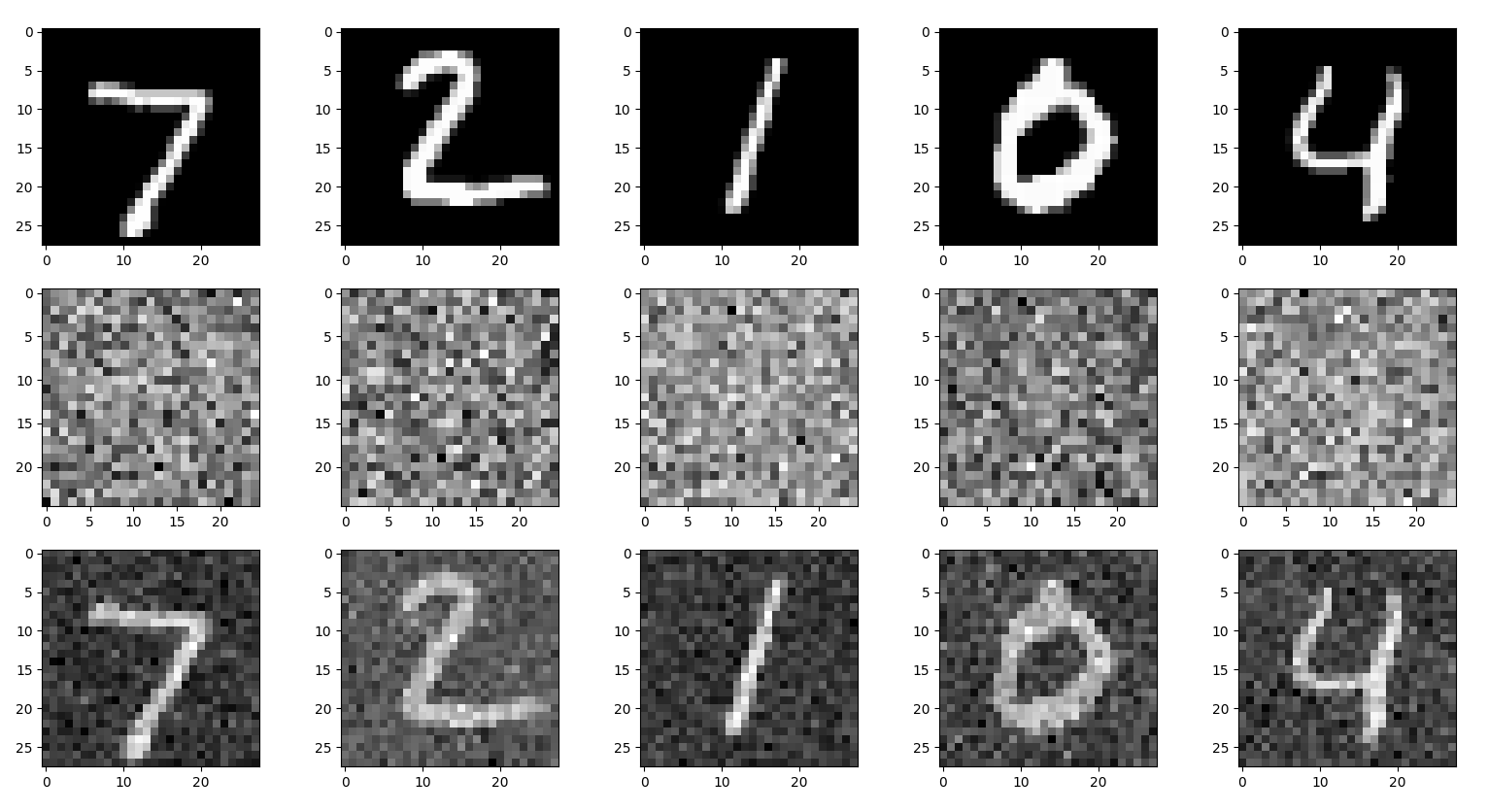}
	\caption{First row shows the original image, second row its projected version $\mathcal{P}_j(x)$, third row the inverse projection $\mathcal{P}^\dagger_j\mathcal{P}_j(x)$. In this example we computed $25\times25$ projections on $28\times28$ MNIST samples.}
	\label{fig:invproj}
\end{figure}

Let $\ell$ be the training loss function.
We propose two possible formulations for the regularization term $\mathcal{R}(\theta)$ of the objective $J(\theta)=\ell(\theta) + \lambda \,\mathcal{R}(\theta)$ on a set of weights $\theta$. The first one, namely $\mathcal{R}_\text{v1}$, adds a penalty which is proportional to the expected norm of the loss gradients computed on the projected data

\begin{align}
 	\mathcal{R}_{\text{v1}} 
 	&= \mathbb{E}_x\bigg[ \mathbb{E}_{\mathcal{P}}\Big[ \big|\big|\nabla_x \ell (f(\mathcal{P}^\dagger\mathcal{P}(x),\theta),y)\big|\big|^2_2 \Big] \bigg]\nonumber\\
	&\approx \frac{1}{np} \sum_{\substack{i=1,\ldots,n\\j=1,\ldots,p}} \Big|\Big| \nabla_x \ell(f(\mathcal{P}^\dagger_j\mathcal{P}_j(x_i),\theta),y_i) \Big|\Big|_2^2.
	\label{eq:v1}
\end{align}

The natural interpretation of $\mathcal{R}_{\text{v1}}$ is that it allows to minimize loss variation across the $\mathcal{P}^\dagger_j\mathcal{P}_j(x_i)$-s.

The second one, $\mathcal{R}_{\text{v2}}$, minimizes the variation of the loss gradients on the original inputs in randomly chosen projected subspaces

\begin{align}
 	\mathcal{R}_{\text{v2}} 
 	&= \mathbb{E}_x\bigg[ \mathbb{E}_{\mathcal{P}}\Big[ \big|\big| \mathcal{P}(\nabla_x \ell (f(x,\theta),y))\big|\big|^2_2 \Big] \bigg]\nonumber\\
	&\approx \frac{1}{np} \sum_{\substack{i=1,\ldots,n\\j=1,\ldots,p}} \Big|\Big| \mathcal{P}_j\Big( \nabla_x \ell(f(x_i,\theta),y_i)\Big) \Big|\Big|_2^2.
	\label{eq:v2}
\end{align}

At each training step we perform a finite approximation of the expectations on minibatches of data, by randomly sampling the directions, the dimension of the projected subspace and the number of projections.

The two regularization terms $\mathcal{R}_{\text{v1}}$ and $\mathcal{R}_{\text{v2}}$ are equivalent as $k\rightarrow \infty$.
\begin{theorem}
Let $\mathcal{R}_{v1}$ and $\mathcal{R}_{v2}$ be the regularization terms defined in Eq.~\ref{eq:v1} and Eq.~\ref{eq:v2}, where $\mathcal{P}:\mathbb{R}^d\rightarrow\mathbb{R}^k$ is a random projection such that the elements of the orthogonal random matrix $R$ are sampled from $\mathcal{N}(0,1/k)$. If $k\in O(d)$ then $\mathcal{R}_{v1}\approx\mathcal{R}_{v2}$ as $k\rightarrow\infty$.
\label{th:equivalence_regularizers}
\end{theorem}

We provide a formal proof of Theorem \ref{th:equivalence_regularizers} in Section \ref{sec:proof_regularizers} of the Appendix.

\section{Experimental Results}
\label{sec:experiments}

We evaluated the proposed methods  on image classification tasks with $10$ classes, using
MNIST~\cite{mnist} and CIFAR-10~\cite{cifar10} dataset. Our baseline models are Convolutional Neural Networks with ReLU activation functions. 
We achieved $99.13\%$ accuracy on MNIST and  $76.52\%$  on CIFAR-10. 
The adversarial attacks in our tests are Fast Gradient Sign Method (FGSM)~\cite{goodfellow2015explaining}, Projected Gradient Descent (PGD)~\cite{kurakin2017physical}, DeepFool~\cite{moosavi2016deepfool}, and Carlini and Wagner (C\&W) in the $L_\infty$ norm~\cite{carliniwagner16towards}. The attacks just mentioned are described in Section \ref{sec:adv_attacks} of the Appendix.
In all such cases, the maximum distance between an image and its adversarial perturbation is set to $\epsilon=0.3$. These methods fall in the white-box category, i.e. they have complete knowledge of their target network. However, due to the transferability property of the attacks, they could also be effective on unknown models. 

Simulations were conducted on a machine with 34 single core Intel(R) Xeon(R) Gold 6140 CPU @ 2.30GHz processors and 200GB of RAM. We made an extensive use of \verb|Tensorflow| \cite{tensorflow} and IBM \verb|adversarial|-\verb|robustness|-\verb|toolbox| \cite{ibm-art} libraries.


\subsection{Adversarial robustness}
Our approach for evaluating the robustness consists in testing the baseline models against several adversarial attacks, using the generated attacks to perform adversarial training on the baselines and, finally, comparing RP-Ensemble and RP-Regularizer with these robust baselines. Such procedure is intended to investigate the generalization capabilities of our methods, which are completely unaware of the chosen attacks, yet compared to models that should exhibit ideal performances against the adversaries, i.e. the adversarially trained ones.

\begin{figure}[!tbp]
	\centering
	\includegraphics[width=\linewidth]{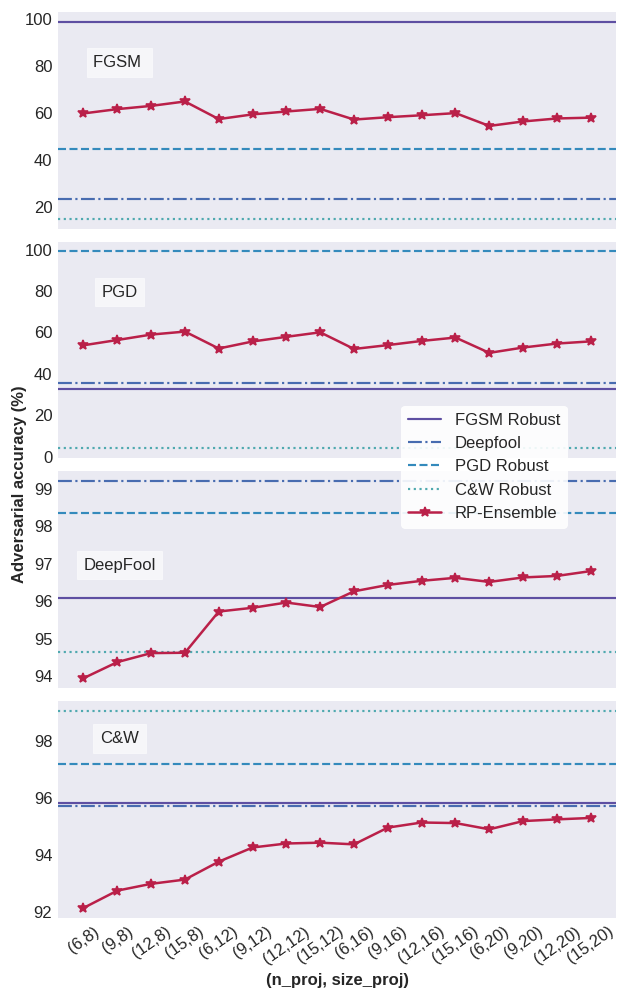}
	\caption{Test accuracy of the baseline, its robust versions and RP-Ensemble model on MNIST dataset. The robust models are the result of adversarial training on the perturbed training sets. RP-Ensemble model is been trained on multiple combinations of number of projections and size of each projection. The evaluations are performed on the original test set and its adversarially perturbed versions.}
	\label{fig:mnist_randens_accuracy}
\end{figure}

\begin{figure}[!h]
	\centering
	\includegraphics[width=0.95\linewidth]{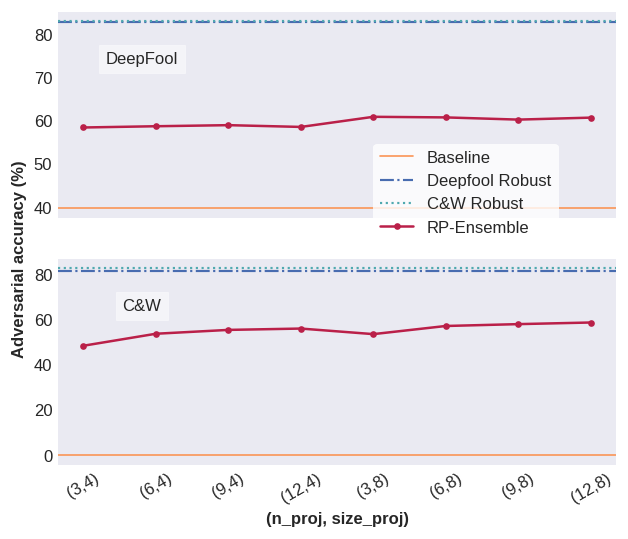}
	\caption{Test accuracy of the baseline, its robust versions and RP-Ensemble model on CIFAR-10 dataset. The robust models are the result of adversarial training on the perturbed training sets. RP-Ensemble model is been trained on multiple combinations of number of projections and size of each projection. The evaluations are performed on the original test set and its adversarially perturbed versions.}
	\label{fig:cifar_randens_accuracy}
\end{figure}

We trained multiple versions of RP-Ensemble, using all the possible combinations of number of projections and size of the projected subspaces shown in Table \ref{tab:ensemble_combinations}. The choice of the classifiers in RP-Ensemble is arbitrary and does not require any model selection step. Indeed, in our experiments each classifier $\psi_j$ is indexed by the seed $j$ used for sampling the projection matrix $R_j$.

We trained a single version of RP-Regularizer on each dataset, by uniformly sampling the number of projections and size of each projection at each training step, as reported in Table \ref{tab:cifar_randreg_accuracy}. We computed the pseudo-inverse matrix $R_j^\dagger$ by using the SVD decomposition of $R_j$, in order to ensure numerical stability.

\begin{table}[!ht]
\caption{Number of projections and size of each \\ projection used for RP-Ensemble.}
\centering
\begin{tabular}{c|c|c}
    \toprule
    \bf Dataset &  \bf Number of projections & \bf Projection size\\
    \midrule
    MNIST & $6,9,12,15$ & $8,12,16,20$\\
    CIFAR-10 & $3,6,9,12$ & $4,8$\\
    \bottomrule
\end{tabular}
\label{tab:ensemble_combinations}
\end{table}

\begin{table}[!ht]
\caption{Number of projections and size of each \\ projection used for RP-Regularizer.}
\centering
\begin{tabular}{c|c|c}
    \toprule
    \bf Dataset &  \bf Number of projections & \bf Projection size\\
    \midrule
    MNIST & \verb|n_proj| $\sim \mathcal{U}(2,8)$ & \verb|size_proj|$ \sim \mathcal{U}(15,25)$\\
    CIFAR-10 & \verb|n_proj|$=1$  & \verb|size_proj|$ \sim \mathcal{U}(5,10)$\\
    \bottomrule
\end{tabular}
\label{tab:regularizer_combinations}
\end{table}

We crafted adversarial perturbations on the original test set using the baseline model, then tested the robustness of RP-Ensemble to the adversaries in terms of prediction accuracy, both on MNIST (Fig. \ref{fig:mnist_randens_accuracy}) and CIFAR-10 (Fig. \ref{fig:cifar_randens_accuracy}). Tables \ref{tab:mnist_randens_accuracy} and \ref{tab:cifar_randens_accuracy} in the Appendix report the exact numerical values for the prediction accuracy. RP-Ensemble brings a general improvement in the adversarial robustness of the baseline model. Adversarially trained robust models show great results on their target attacks but perform poorly on the other ones, while RP-Ensemble preserves its robustness across the different attacks.

\begin{figure}[!htbp]
	\centering
	\includegraphics[width=\linewidth]{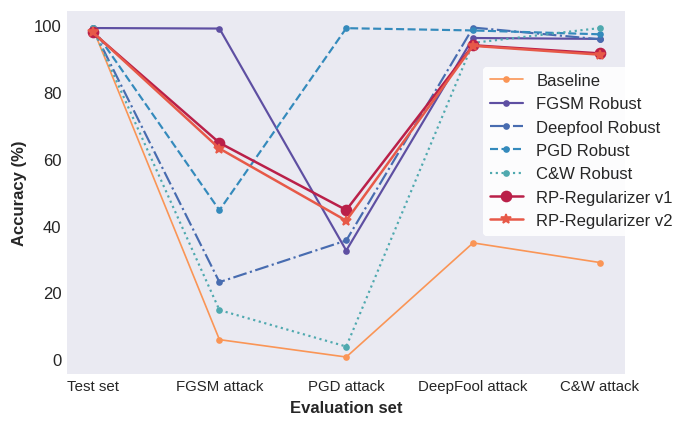}
	\caption{Test accuracy of RP-Regularizer on MNIST. We compare the baseline model, the adversarially trained robust models and two different versions of RP-Regularizer model, namely  $\mathcal{R}_\text{v1}$ (eq. \ref{eq:v1}) and $\mathcal{R}_\text{v2}$ (eq. \ref{eq:v2}).
    Adversarial perturbations are produced on the baseline model using FGSM, PGD, Deepfool and Carlini \& Wagner attacks. }
	\label{fig:mnist_randreg_accuracy}
\end{figure}

\begin{figure}[!htbp]
	\centering
    	\includegraphics[width=0.95\linewidth]{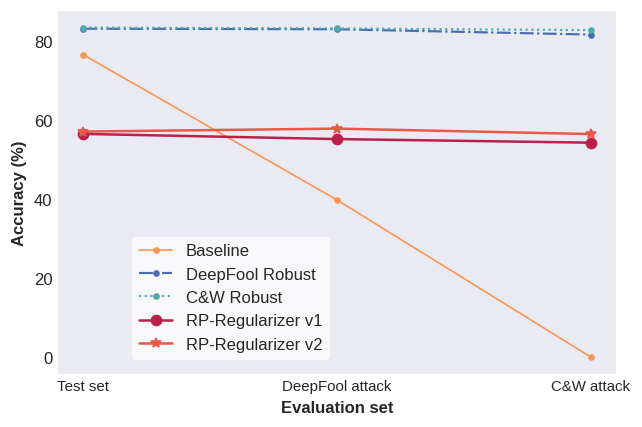}
	\caption{Test accuracy of RP-Regularizer on CIFAR-10. We compare the baseline model, the adversarially trained robust models and two different versions of RP-Regularizer model, namely  $\mathcal{R}_\text{v1}$ (eq. \ref{eq:v1}) and $\mathcal{R}_\text{v2}$ (eq. \ref{eq:v2}).
Adversarial perturbations are produced on the baseline model using FGSM, PGD, Deepfool and Carlini \& Wagner attacks. }	
	\label{fig:cifar_randreg_accuracy}
\end{figure}

RP-Regularizer is able to reach competitive performances in comparison to the SOTA models on MNIST (Fig. \ref{fig:mnist_randreg_accuracy}). Prediction accuracies are higher on DeepFool and Carlini \& Wagner attacks than on FGSM and PGD, suggesting that this method performs better on algorithms which are optimized to produce perturbations that are closer to the original samples (e.g. C\&W), rather than faster in computation (e.g. FGSM). 
The results are less striking on CIFAR-10 (Fig. \ref{fig:cifar_randreg_accuracy}), but we stress that the robustness of RP-Regularizer improves as the number of projections increases and that on CIFAR-10 we kept it low (always equal to $1$) to maintain a balance between  computational efficiency and adversarial accuracy.
The trade-off between these two objectives needs to be further explored. In RP-Regularizer $\ell$ is a crossentropy loss function.

\subsection{Computational efficiency of RP-Ensemble}

Classifiers $\psi_j$ from PR-Ensemble are defined in independent projected subspaces, thus their training can be efficiently parallelized.
This allows to keep its training time close to that of the baseline, or even lower when choosing a small number of projections
(Fig. \ref{fig:randens_training_complexity}).

Moreover, pairwise distances between the projected points are nearly preserved, so the projected versions of the original images contain most of the original information. This implies that the projection classifiers $\psi_j$ are computationally efficient, due to the dimensionality reduction of their inputs, and are also able to learn features which turn out to be significant as a defense against the attacks.

\begin{figure}[!ht]
	\centering
	\includegraphics[width=\linewidth]{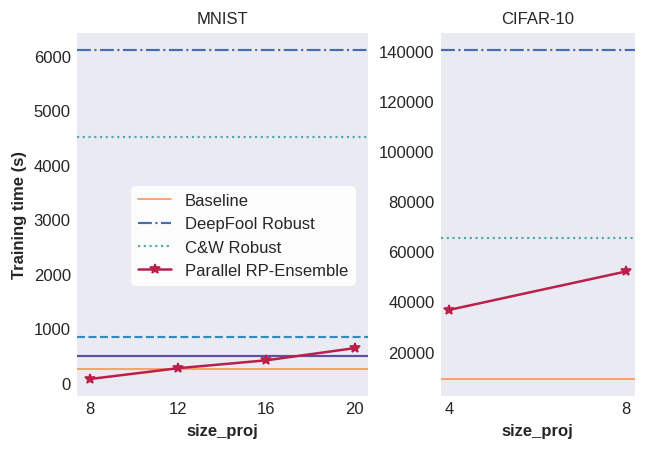}
	\caption{
	Training time of RP-Ensemble on MNIST and CIFAR-10. We compare its efficiency to that of the baseline models and the adversarially trained robust models. 
	}
	\label{fig:randens_training_complexity}
\end{figure}
\section{Related Work}
\label{sec:related_work}

\subsection{Adversarial training}
One of the simplest and most effective approaches for learning robust models is \emph{adversarial training} \cite{goodfellow2015explaining}. This process consists in training a classifier by including adversarial examples in the training data, thus it allows to directly convert any attack into a defense. 
The biggest limitation of this approach is that it tends to overfit the chosen attack, meaning that the adversarially trained model perform well on the attacks which they learned to defend from, but might show poor transferability on other threats. Adversarial training could be interpreted as a form of data augmentation, which significantly differs from the traditional approach: instead of applying transformations that are expected to occur in the test set (translations, rotations, etc.), only the most unlikely examples are added. This method corresponds to a dilation of the manifold: adversarial examples are learned in a halo around the surface, which makes the manifold smoother \cite{dube2018highdim}. In this regard, it should be noted that RP-Ensemble does not perform any data augmentation in the original high dimensional space, since the projected data samples lie in new subspaces. RP-Regularizer, instead, produces new high dimensional examples, which could be formulated as perturbations of the original ones. 

In~\cite{finlay19improved} Finlay et al. show that regularization of the loss gradient on the inputs improves adversarial robustness. In particular, they
notice that the \emph{Total Variation} regularization \cite{rudin1992totalvariation} can be interpreted as the regularization induced by a single step of adversarial training on gradient-based attacks.

\subsection{Randomization in a high codimension setting}

Randomization has been proven effective as a defense \cite{he2019parametric, araujo2019robust}, a detection \cite{drenkow2020random} and a regularization \cite{durrant2015random} strategy against adversarial attacks. Xie et al. \cite{xie2017mitigating} apply two random transformations to the input images,
Liu, Cheng, Zhang and Hsieh \cite{liu2018towards} add random noise between the layers of the architecture, Dhillon et al. \cite{dhillon2018stochastic} randomly prune activations between the layers, Xu, Evans and Qi \cite{xu2017feature} and Liao and Wagner \cite{liao2018defense} propose input denoising and feature denoising methods.
The reasoning behind these techniques is that NNs are usually robust to random perturbations \cite{ren2020adversarial}, thus incorporating them in the models might weaken adversarial perturbations .

We also explore the geometrical properties related to randomization. Khuory et al. \cite{khuory18geometry} first highlighted the role of codimension in the generation of adversarial examples. Their analysis suggests that  adversarial perturbations mainly arise in the directions that are normal to the data manifold, so as the codimension in the embedding space increases there is a higher number of directions in which one could build adversarial perturbations. Our framework is strongly influenced by this finding, as it suggests that by randomly selecting directions it should be more likely to catch features that are significant in the adversarial context.

\subsection{Bayesian interpretation of ensembles}

Recent findings suggests a connection between ensembles of NNs and Bayesian NNs, where the goal is to learn the posterior distribution on the weights and use it to perform predictive inference on new observations.
Lakshminarayanan, Pritzel and Blundell presented \emph{Deep Ensembles} \cite{lakshminarayanan2016simple} as a computationally cheap alternative to Bayesian NNs. Dropout has been used to estimate the predictive uncertainty in \emph{MC-Dropout} ensembles \cite{gal2016dropout}. 
Latest research also shows that Bayesian inference is effective at learning adversarially robust models \cite{bekasov2018bayesian, gal2018sufficient}. Moreover, it has been proved that, under specific theoretical assumptions, Bayesian NNs are robust to gradient-based attacks \cite{carbone2020robustness}.

\subsection{Robustness in a chosen norm}
The robustness of classifiers is strongly related to the geometry of the learned decision boundaries. In fact, in order to learn robust decision boundaries, a model has to correctly classify all the input points lying in a neighbourhood of the data manifold. In particular, adversarially perturbed points always lie extremely close to the decision boundaries \cite{dube2018highdim}.
But robustness conditions change under different $p$-norms, meaning that no single decision boundary can be optimally robust in all norms \cite{khuory18geometry}.
E.g. if a classifier is trained to be robust under $L_\infty$ norm, poor robustness under the $L_2$ norm should be expected. 
In general, no distance metric can be considered a perfect measure of similarity \cite{carliniwagner16towards}, so one of the strengths of our methods is that they are totally independent on the norm chosen for the attacks.
\section{Conclusions}

Adversarial examples show that many of the modern machine learning algorithms can be fooled in unexpected ways. Both in terms of attacks and defenses, many theoretical problems still remain open. From a practical point of view, no one has yet designed a powerful defense algorithm which could be suitable against a variety of attacks, with different degrees of knowledge about models under attack and their predictions. The most effective defense techniques, e.g. adversarial training, are still too computationally expensive.

We empirically showed that random projections of the training data act as attack-independent adversarial features, that can be used to provide better resilience to adversarial perturbations. 
We proposed a fine-tuning method and a regularization method, both based on the computation of random projections of the inputs. As future work we plan to improve the computational cost of RP-Regularizer and to compare the performances of our methods to that of other attack-independent defense strategies. We believe that a further exploration of the connections between random projections and the geometrical characterization of adversarial regions could bring valuable insights to adversarial defense research.

\bibliographystyle{unsrtnat}
\bibliography{bibliography}

\begin{thebibliography}{38}
\providecommand{\natexlab}[1]{#1}
\providecommand{\url}[1]{\texttt{#1}}
\expandafter\ifx\csname urlstyle\endcsname\relax
  \providecommand{\doi}[1]{doi: #1}\else
  \providecommand{\doi}{doi: \begingroup \urlstyle{rm}\Url}\fi

\bibitem[Szegedy et~al.(2014)Szegedy, Zaremba, Sutskever, Bruna, Erhan,
  Goodfellow, and Fergus]{szegedy2014intriguing}
Christian Szegedy, Wojciech Zaremba, Ilya Sutskever, Joan Bruna, Dumitru Erhan,
  Ian Goodfellow, and Rob Fergus.
\newblock Intriguing properties of neural networks, 2014.

\bibitem[Kurakin et~al.(2016)Kurakin, Goodfellow, and
  Bengio]{kurakin2017physical}
Alexey Kurakin, Ian~J. Goodfellow, and Samy Bengio.
\newblock Adversarial examples in the physical world.
\newblock \emph{CoRR}, abs/1607.02533, 2016.
\newblock URL \url{http://arxiv.org/abs/1607.02533}.

\bibitem[Carlini et~al.(2019)Carlini, Athalye, Papernot, Brendel, Rauber,
  Tsipras, Goodfellow, Madry, and Kurakin]{carlini2019onevaluating}
Nicholas Carlini, Anish Athalye, Nicolas Papernot, Wieland Brendel, Jonas
  Rauber, Dimitris Tsipras, Ian~J. Goodfellow, Aleksander Madry, and Alexey
  Kurakin.
\newblock On evaluating adversarial robustness.
\newblock \emph{CoRR}, abs/1902.06705, 2019.
\newblock URL \url{http://arxiv.org/abs/1902.06705}.

\bibitem[Evtimov et~al.(2017)Evtimov, Eykholt, Fernandes, Kohno, Li, Prakash,
  Rahmati, and Song]{eykholt17robust}
Ivan Evtimov, Kevin Eykholt, Earlence Fernandes, Tadayoshi Kohno, Bo~Li, Atul
  Prakash, Amir Rahmati, and Dawn Song.
\newblock Robust physical-world attacks on machine learning models.
\newblock \emph{CoRR}, abs/1707.08945, 2017.
\newblock URL \url{http://arxiv.org/abs/1707.08945}.

\bibitem[Fawzi et~al.(2018)Fawzi, Fawzi, and Frossard]{fawzi2018analysis}
Alhussein Fawzi, Omar Fawzi, and Pascal Frossard.
\newblock Analysis of classifiers’ robustness to adversarial perturbations.
\newblock \emph{Machine Learning}, 107\penalty0 (3):\penalty0 481--508, 2018.

\bibitem[Goodfellow et~al.(2015)Goodfellow, Shlens, and
  Szegedy]{goodfellow2015explaining}
Ian Goodfellow, Jonathon Shlens, and Christian Szegedy.
\newblock Explaining and harnessing adversarial examples.
\newblock \emph{International Conference on Learning Representations}, 2015.
\newblock URL \url{http://arxiv.org/abs/1412.6572}.

\bibitem[Genovese et~al.(2012)Genovese, Perone-Pacifico, Verdinelli, Wasserman,
  et~al.]{genovese2012manifold}
Christopher~R Genovese, Marco Perone-Pacifico, Isabella Verdinelli, Larry
  Wasserman, et~al.
\newblock Manifold estimation and singular deconvolution under hausdorff loss.
\newblock \emph{The Annals of Statistics}, 40\penalty0 (2):\penalty0 941--963,
  2012.

\bibitem[Dasgupta and Freund(2008)]{dasgupta2008random}
Sanjoy Dasgupta and Yoav Freund.
\newblock Random projection trees and low dimensional manifolds.
\newblock In \emph{Proceedings of the fortieth annual ACM symposium on Theory
  of computing}, pages 537--546, 2008.

\bibitem[Donoho and Grimes(2003)]{donoho2003hessian}
David~L Donoho and Carrie Grimes.
\newblock Hessian eigenmaps: Locally linear embedding techniques for
  high-dimensional data.
\newblock \emph{Proceedings of the National Academy of Sciences}, 100\penalty0
  (10):\penalty0 5591--5596, 2003.

\bibitem[Khoury and Hadfield{-}Menell(2018)]{khuory18geometry}
Marc Khoury and Dylan Hadfield{-}Menell.
\newblock On the geometry of adversarial examples.
\newblock \emph{CoRR}, abs/1811.00525, 2018.
\newblock URL \url{http://arxiv.org/abs/1811.00525}.

\bibitem[Dasgupta and Gupta(1999)]{johnson-lind}
Sanjoy Dasgupta and Anupam Gupta.
\newblock An elementary proof of the johnson-lindenstrauss lemma.
\newblock Technical report, 1999.

\bibitem[Kaski(1998)]{kaski98}
Samuel Kaski.
\newblock Dimensionality reduction by random mapping: Fast similarity
  computation for clustering.
\newblock volume~1, pages 413 -- 418 vol.1, 06 1998.
\newblock ISBN 0-7803-4859-1.
\newblock \doi{10.1109/IJCNN.1998.682302}.

\bibitem[Rudin et~al.(1992)Rudin, Osher, and Fatemi]{rudin1992totalvariation}
Leonid~I Rudin, Stanley Osher, and Emad Fatemi.
\newblock Nonlinear total variation based noise removal algorithms.
\newblock \emph{Physica D: nonlinear phenomena}, 60\penalty0 (1-4):\penalty0
  259--268, 1992.

\bibitem[Finlay et~al.(2019)Finlay, Oberman, and Abbasi]{finlay19improved}
Chris Finlay, Adam~M. Oberman, and Bilal Abbasi.
\newblock Improved robustness to adversarial examples using lipschitz
  regularization of the loss, 2019.

\bibitem[Penrose(1955)]{penrose_1955}
R.~Penrose.
\newblock A generalized inverse for matrices.
\newblock \emph{Mathematical Proceedings of the Cambridge Philosophical
  Society}, 51\penalty0 (3):\penalty0 406–413, 1955.
\newblock \doi{10.1017/S0305004100030401}.

\bibitem[LeCun and Cortes(2010)]{mnist}
Yann LeCun and Corinna Cortes.
\newblock {MNIST} handwritten digit database.
\newblock 2010.
\newblock URL \url{http://yann.lecun.com/exdb/mnist/}.

\bibitem[Krizhevsky et~al.()Krizhevsky, Nair, and Hinton]{cifar10}
Alex Krizhevsky, Vinod Nair, and Geoffrey Hinton.
\newblock Cifar-10 (canadian institute for advanced research).
\newblock URL \url{http://www.cs.toronto.edu/~kriz/cifar.html}.

\bibitem[Moosavi-Dezfooli et~al.(2016)Moosavi-Dezfooli, Fawzi, and
  Frossard]{moosavi2016deepfool}
Seyed-Mohsen Moosavi-Dezfooli, Alhussein Fawzi, and Pascal Frossard.
\newblock Deepfool: a simple and accurate method to fool deep neural networks.
\newblock In \emph{Proceedings of the IEEE conference on computer vision and
  pattern recognition}, pages 2574--2582, 2016.

\bibitem[Carlini and Wagner(2016)]{carliniwagner16towards}
Nicholas Carlini and David~A. Wagner.
\newblock Towards evaluating the robustness of neural networks.
\newblock \emph{CoRR}, abs/1608.04644, 2016.
\newblock URL \url{http://arxiv.org/abs/1608.04644}.

\bibitem[Abadi et~al.(2015)Abadi, Agarwal, Barham, Brevdo, Chen, Citro,
  Corrado, Davis, Dean, Devin, Ghemawat, Goodfellow, Harp, Irving, Isard, Jia,
  Jozefowicz, Kaiser, Kudlur, Levenberg, Man\'{e}, Monga, Moore, Murray, Olah,
  Schuster, Shlens, Steiner, Sutskever, Talwar, Tucker, Vanhoucke, Vasudevan,
  Vi\'{e}gas, Vinyals, Warden, Wattenberg, Wicke, Yu, and Zheng]{tensorflow}
Mart\'{\i}n Abadi, Ashish Agarwal, Paul Barham, Eugene Brevdo, Zhifeng Chen,
  Craig Citro, Greg~S. Corrado, Andy Davis, Jeffrey Dean, Matthieu Devin,
  Sanjay Ghemawat, Ian Goodfellow, Andrew Harp, Geoffrey Irving, Michael Isard,
  Yangqing Jia, Rafal Jozefowicz, Lukasz Kaiser, Manjunath Kudlur, Josh
  Levenberg, Dan Man\'{e}, Rajat Monga, Sherry Moore, Derek Murray, Chris Olah,
  Mike Schuster, Jonathon Shlens, Benoit Steiner, Ilya Sutskever, Kunal Talwar,
  Paul Tucker, Vincent Vanhoucke, Vijay Vasudevan, Fernanda Vi\'{e}gas, Oriol
  Vinyals, Pete Warden, Martin Wattenberg, Martin Wicke, Yuan Yu, and Xiaoqiang
  Zheng.
\newblock {TensorFlow}: Large-scale machine learning on heterogeneous systems,
  2015.
\newblock URL \url{http://tensorflow.org/}.
\newblock Software available from tensorflow.org.

\bibitem[Nicolae et~al.(2018)Nicolae, Sinn, Minh, Rawat, Wistuba, Zantedeschi,
  Molloy, and Edwards]{ibm-art}
Maria{-}Irina Nicolae, Mathieu Sinn, Tran~Ngoc Minh, Ambrish Rawat, Martin
  Wistuba, Valentina Zantedeschi, Ian~M. Molloy, and Benjamin Edwards.
\newblock Adversarial robustness toolbox v0.2.2.
\newblock \emph{CoRR}, abs/1807.01069, 2018.
\newblock URL \url{http://arxiv.org/abs/1807.01069}.

\bibitem[Dube(2018)]{dube2018highdim}
Simant Dube.
\newblock High dimensional spaces, deep learning and adversarial examples.
\newblock \emph{CoRR}, abs/1801.00634, 2018.
\newblock URL \url{http://arxiv.org/abs/1801.00634}.

\bibitem[He et~al.(2019)He, Rakin, and Fan]{he2019parametric}
Zhezhi He, Adnan~Siraj Rakin, and Deliang Fan.
\newblock Parametric noise injection: Trainable randomness to improve deep
  neural network robustness against adversarial attack.
\newblock In \emph{Proceedings of the IEEE/CVF Conference on Computer Vision
  and Pattern Recognition}, pages 588--597, 2019.

\bibitem[Araujo et~al.(2019)Araujo, Meunier, Pinot, and
  Negrevergne]{araujo2019robust}
Alexandre Araujo, Laurent Meunier, Rafael Pinot, and Benjamin Negrevergne.
\newblock Robust neural networks using randomized adversarial training.
\newblock \emph{arXiv preprint arXiv:1903.10219}, 2019.

\bibitem[Drenkow et~al.(2020)Drenkow, Fendley, and Burlina]{drenkow2020random}
Nathan Drenkow, Neil Fendley, and Philippe Burlina.
\newblock Random projections for adversarial attack detection.
\newblock \emph{arXiv preprint arXiv:2012.06405}, 2020.

\bibitem[Durrant and Kab{\'a}n(2015)]{durrant2015random}
Robert~J Durrant and Ata Kab{\'a}n.
\newblock Random projections as regularizers: learning a linear discriminant
  from fewer observations than dimensions.
\newblock \emph{Machine Learning}, 99\penalty0 (2):\penalty0 257--286, 2015.

\bibitem[Xie et~al.(2017)Xie, Wang, Zhang, Ren, and Yuille]{xie2017mitigating}
Cihang Xie, Jianyu Wang, Zhishuai Zhang, Zhou Ren, and Alan Yuille.
\newblock Mitigating adversarial effects through randomization.
\newblock \emph{arXiv preprint arXiv:1711.01991}, 2017.

\bibitem[Liu et~al.(2018)Liu, Cheng, Zhang, and Hsieh]{liu2018towards}
Xuanqing Liu, Minhao Cheng, Huan Zhang, and Cho-Jui Hsieh.
\newblock Towards robust neural networks via random self-ensemble.
\newblock In \emph{Proceedings of the European Conference on Computer Vision
  (ECCV)}, pages 369--385, 2018.

\bibitem[Dhillon et~al.(2018)Dhillon, Azizzadenesheli, Lipton, Bernstein,
  Kossaifi, Khanna, and Anandkumar]{dhillon2018stochastic}
Guneet~S Dhillon, Kamyar Azizzadenesheli, Zachary~C Lipton, Jeremy Bernstein,
  Jean Kossaifi, Aran Khanna, and Anima Anandkumar.
\newblock Stochastic activation pruning for robust adversarial defense.
\newblock \emph{arXiv preprint arXiv:1803.01442}, 2018.

\bibitem[Xu et~al.(2017)Xu, Evans, and Qi]{xu2017feature}
Weilin Xu, David Evans, and Yanjun Qi.
\newblock Feature squeezing: Detecting adversarial examples in deep neural
  networks.
\newblock \emph{arXiv preprint arXiv:1704.01155}, 2017.

\bibitem[Liao et~al.(2018)Liao, Liang, Dong, Pang, Hu, and
  Zhu]{liao2018defense}
Fangzhou Liao, Ming Liang, Yinpeng Dong, Tianyu Pang, Xiaolin Hu, and Jun Zhu.
\newblock Defense against adversarial attacks using high-level representation
  guided denoiser.
\newblock In \emph{Proceedings of the IEEE Conference on Computer Vision and
  Pattern Recognition}, pages 1778--1787, 2018.

\bibitem[Ren et~al.(2020)Ren, Zheng, Qin, and Liu]{ren2020adversarial}
Kui Ren, Tianhang Zheng, Zhan Qin, and Xue Liu.
\newblock Adversarial attacks and defenses in deep learning.
\newblock \emph{Engineering}, 6\penalty0 (3):\penalty0 346--360, 2020.

\bibitem[Lakshminarayanan et~al.(2016)Lakshminarayanan, Pritzel, and
  Blundell]{lakshminarayanan2016simple}
Balaji Lakshminarayanan, Alexander Pritzel, and Charles Blundell.
\newblock Simple and scalable predictive uncertainty estimation using deep
  ensembles.
\newblock \emph{arXiv preprint arXiv:1612.01474}, 2016.

\bibitem[Gal and Ghahramani(2016)]{gal2016dropout}
Yarin Gal and Zoubin Ghahramani.
\newblock Dropout as a bayesian approximation: Representing model uncertainty
  in deep learning.
\newblock In \emph{international conference on machine learning}, pages
  1050--1059. PMLR, 2016.

\bibitem[Bekasov and Murray(2018)]{bekasov2018bayesian}
Artur Bekasov and Iain Murray.
\newblock Bayesian adversarial spheres: Bayesian inference and adversarial
  examples in a noiseless setting.
\newblock \emph{arXiv preprint arXiv:1811.12335}, 2018.

\bibitem[Gal and Smith(2018)]{gal2018sufficient}
Yarin Gal and Lewis Smith.
\newblock Sufficient conditions for idealised models to have no adversarial
  examples: a theoretical and empirical study with bayesian neural networks.
\newblock \emph{arXiv preprint arXiv:1806.00667}, 2018.

\bibitem[Carbone et~al.(2020)Carbone, Wicker, Laurenti, Patane, Bortolussi, and
  Sanguinetti]{carbone2020robustness}
Ginevra Carbone, Matthew Wicker, Luca Laurenti, Andrea Patane, Luca Bortolussi,
  and Guido Sanguinetti.
\newblock Robustness of bayesian neural networks to gradient-based attacks,
  2020.

\bibitem[Vershynin(2018)]{vershynin2018high}
Roman Vershynin.
\newblock \emph{High-dimensional probability: An introduction with applications
  in data science}, volume~47.
\newblock Cambridge university press, 2018.

\end{thebibliography}
\newpage
\section{Appendix}

\subsection{Prediction accuracy}

\begin{table}[!htbp]
\caption{Test accuracy of the baseline, its robust versions and RP-Ensemble model on MNIST dataset. The robust models are the result of adversarial training on the perturbed training sets. RP-Ensemble model is been trained on multiple combinations of number of projections and size of each projection. The evaluations are performed on the original test set and its adversarially perturbed versions.}
\center
	\begin{tabular}{ c | c c c c c }
	\toprule
		\begin{tabular}{@{}c@{}}
    		Prediction\\ accuracy (\%)
    		\end{tabular}  
    	& Test set & FGSM & PGD & DeepFool & C\&W
    	\\
	\midrule
	\midrule
	Baseline & $99.13$ & $5.91$ & $0.71$ & $34.83$ & $28.96$\\
	\midrule
	\multicolumn{6}{c}{Adversarially trained models}\\
	\cline{2-6} \\ [-1.5ex]
	FGSM &  $99.13$ & $\mathbf{98.91}$ & $32.55$ & $96.08$ & $95.82$ \\
	PGD & $99.10$ & $44.60$ & $\mathbf{99.02}$ & $98.34$ & $97.19$ \\
	DeepFool & $99.03$ & $23.11$ & $35.55$ & $\mathbf{99.20}$ & ${95.70}$\\
	C\&W & $99.10$ & ${14.74}$ & ${3.85}$ & ${94.63}$ & $\mathbf{99.06}$\\
	\midrule
	\multicolumn{6}{c}{RP-Ensemble on (n\_proj, size\_proj) combinations}\\
	\cline{2-6} \\ [-1.6ex]
	$(6, 8)$  & $97.66$ & $59.71$  & $53.68$  & $93.94$  & $92.12$\\    
$(9, 8)$  & $97.57$ & $61.55$ & $56.23$  & $94.37$  & $92.73$\\   
$(12, 8)$  & $97.45$ & $62.93$  & $58.85$  & $94.61$  & $92.97$\\   
$(15, 8)$  & $97.47$ & $\mathbf{64.82}$ & $\mathbf{60.30}$  & $94.62$  & $93.12$\\
$(6  , 12 )$  &      $98.12$ & $57.29$  & $52.12$  & $95.72$  & $93.75$\\    
$(9  , 12 )$  & $98.02$ & $59.30$  & $55.52$  & $95.82$  & $94.25$\\
$(12 , 12 )$  & $ 97.97$ & $60.54$  & $57.77$  & $95.96$  & $94.39$\\
(15 , 12 )  &  $97.91$ & $61.65$  & $59.95$  & $95.84$  & $94.42$\\    
$(6  , 16 )$  &      $98.22$ & $57.09$  & $51.90$  & $96.26$  & $94.36$\\    
$(9  , 16 )$  & $98.33$ & $58.10$  & $53.77$  & $96.43$  & $94.95$\\    
$(12 , 16 )$  &      $98.32$ & $58.93$  & $55.78$  & $96.54$  & $95.13$\\    
$(15 , 16 )$  & $98.26$ & $59.85$  & $57.42$  & $96.62$  & $95.11$\\
$(6  , 20 )$  &      $98.49$ & $54.37$  & $50.02$  & $96.51$  & $94.89$\\    
$(9  , 20 )$  &      $98.42$ & $56.28$  & $52.60$  & $96.63$  & $95.18$\\    
$(12 , 20 )$  &     $98.40$ & $57.56$  & $54.52$  & $96.67$  & $95.24$\\    
$(15 , 20 )$  &      $98.40$ & $57.91$  & $55.57$  & $\mathbf{96.80}$  & $\mathbf{95.29}$\\    
	\bottomrule
\end{tabular}

\label{tab:mnist_randens_accuracy}
\end{table}

\begin{table}[!ht]
\caption{Test accuracy of the baseline, its robust versions and RP-Ensemble model on CIFAR-10 dataset. The robust models are the result of adversarial training on the perturbed training sets. RP-Ensemble model is been trained on multiple combinations of number of projections and size of each projection. The evaluations are performed on the original test set and its adversarially perturbed versions.}
\center
	\begin{tabular}{ c |c c c}
	\toprule
		\begin{tabular}{@{}c@{}}
    		Prediction\\ accuracy (\%)
    		\end{tabular}  
    	& Test set & DeepFool & C\&W
    	\\
	\midrule
	\midrule
	Baseline & $76.52$ &  $39.77$ &  $0.00$\\
	\midrule
	\multicolumn{4}{c}{Adversarially trained models}\\
	\cline{2-4}  \\ [-1.6ex]
	DeepFool & $83.16$ &  $\mathbf{83.01}$  & $81.67$  \\
	C\&W & $83.44$  & $83.23$  & $\mathbf{82.79}$ \\
	\midrule
	\multicolumn{4}{c}{
		\begin{tabular}{@{}c@{}}
    		RP-Ensemble model 	
    		  on (n\_proj, size\_proj)
    		\end{tabular} 
	
	}\\
	\cline{2-4}\\ [-1.6ex]
	$(3, 4)$ & $\mathbf{67.93}$ & $58.52$ & $48.48$ \\
    $(6, 4)$ & $64.59$ & $58.81$ & $53.78$ \\
    $(9, 4)$ & $63.15$ & $59.06$ & $55.48$ \\
    $(12, 4)$ & $61.93$ & $58.65$ & $56.07$ \\
	$(3, 8)$ & $67.66$ &  $\mathbf{61.00}$ &  $53.61$ \\		
	$(6, 8)$  & $64.83$ &  $60.86$ &  $57.21$  \\   
	$(9, 8)$ & $63.36$ &  $60.35$ &  $58.03$\\   
	$(12, 8)$ & $62.99$ &  $60.81$ &  $\mathbf{58.75}$ \\  
	\bottomrule
\end{tabular}

\label{tab:cifar_randens_accuracy}
\end{table}

\begin{table}[!ht]
\caption{
Test accuracy of RP-Regularizer on MNIST. We compare the baseline model, the adversarially trained robust models and two different versions of RP-Regularizer model, namely  $\mathcal{R}_\text{v1}$ (\ref{eq:v1}) and $\mathcal{R}_\text{v2}$ (\ref{eq:v2}).
Adversarial perturbations are produced on the baseline model using FGSM, PGD, Deepfool and Carlini \& Wagner attacks. 
}
\center
	\begin{tabular}{ c |c c c c c }
	\toprule
		\begin{tabular}{@{}c@{}}
    		Prediction\\ accuracy (\%)
    		\end{tabular}  
    	& Test set & FGSM & PGD & DeepFool & C\&W
    	\\
	\midrule
	\midrule
	Baseline & $99.13$ & $5.91$ & $0.71$ & $34.83$ & $28.96$\\
	\midrule
	\multicolumn{6}{c}{Adversarially trained models}\\
	\cline{2-6} \\ [-1.6ex]
	FGSM &  $99.13$ & $\mathbf{98.91}$ & $32.55$ & $96.08$ & $95.82$ \\
	PGD &  $99.10$ & $44.60$ & $\mathbf{99.02}$ & $98.34$ & $97.19$ \\
	DeepFool & $99.03$ & $23.11$ & $35.55$ & $\mathbf{99.20}$ & $95.70$\\
	C\&W & $99.10$ & ${14.74}$ & ${3.85}$ & ${94.63}$ & $\mathbf{99.06}$\\
	\midrule

	\multicolumn{6}{c}{RP-Regularizer model}\\

	\cline{2-6} \\ [-1.6ex]
	$\mathcal{R}_\text{v1}$, $\lambda=0.4$ & $97.92$&  $62.34$&  $38.96$& $93.24$& $90.75$ \\
	$\mathcal{R}_\text{v2}$, $\lambda=0.4$ &$97.82$& $63.25$&   $42.37$&  $93.86$&  $91.56$ \\ 
$\mathcal{R}_\text{v1}$, $\lambda=0.5$ & $97.53$& $\mathbf{69.12}$&  $\mathbf{52.39}$& $\mathbf{94.44}$& $\mathbf{91.96}$ \\
$\mathcal{R}_\text{v2}$, $\lambda=0.5$  & $98.06$& $60.64$ & $36.28$& $93.92$&   $91.05$ \\ 
	$\mathcal{R}_\text{v1}$, $\lambda=0.6$ & $97.80$ &$62.78$ & $42.61$ & $94.05$ &$91.73$\\
	$\mathcal{R}_\text{v2}$, $\lambda=0.6$ &$97.69$& $65.25$ & $45.77$&     $93.45$&  $90.70$  	\\
	
	\bottomrule
	\end{tabular}

\label{tab:mnist_randreg_accuracy}
\end{table}

\begin{table}[!ht]
\caption{Test accuracy of RP-Regularizer on MNIST.
We compare the baseline model, the adversarially trained robust models and two different versions of RP-Regularizer model, namely  $\mathcal{R}_\text{v1}$ (\ref{eq:v1}) and $\mathcal{R}_\text{v2}$ (\ref{eq:v2}).
Adversarial perturbations are produced on the baseline model using Deepfool and Carlini \& Wagner attacks. }
\center
	\begin{tabular}{ c | c c c}
	\toprule	\begin{tabular}{@{}c@{}}
    		Prediction\\ accuracy (\%)
    		\end{tabular}  
    	& Test set & DeepFool & C\&W	\\
	\midrule
	\midrule
	Baseline & 76.52  & 39.77  & 0.00  \\ 
	\midrule
	\multicolumn{4}{c}{Adversarially trained models}\\
	\cline{2-4} \\ [-1.6ex]
	DeepFool & 83.16  & 83.01  & 81.67 \\ 
	C\&W  & 83.44  & 83.23  & 82.79 \\
	\midrule
	\multicolumn{4}{c}{RP-Regularizer model, $\lambda=0.5$}\\
	\cline{2-4} \\ [-1.6ex]
 	$\mathcal{R}_\text{v1}$ & 56.51 &  55.20 & 54.29  \\ 
 	$\mathcal{R}_\text{v2}$ & 57.10 &  57.85 & 56.47   \\
	\bottomrule
	\end{tabular}
			    						
\label{tab:cifar_randreg_accuracy}
\end{table}

\subsection{Adversarial attacks}
\label{sec:adv_attacks}

\emph{Fast Gradient Sign Method} (FGSM)~\cite{goodfellow2015explaining} is an untargeted attack, i.e. it does not push the misclassification to any specific class. FGSM adds a fixed noise to the input $x$ in the direction of the loss gradient w.r.t. $x$ 
$$
\tilde{x}=x + \epsilon \,\text{sgn}  \nabla_x \ell(f(x,\theta),y).
$$
It is a popular choice for adversarial training on a large number of samples due to its computational efficiency, since
it only requires one gradient evaluation for each given input.

\vspace{3pt}

\emph{Projected Gradient Descent} (PGD)~\cite{kurakin2017physical} is an iterative attack which starts from a random perturbation $\tilde{x}_0$ of $x$ in an $\epsilon\text{-}L_\infty$ ball around the input
sample. At each step, it performs an FGSM attack with a smaller step size $\delta<\epsilon$ and projects the attack back in the $\epsilon\text{-}L_\infty$ ball
$$
\tilde{x}_{t+1} = \text{Proj}\{\tilde{x}_{t}+\delta \, \text{sgn} \nabla_x \ell(f(\tilde{x}_t,\theta),y)\}
$$

\vspace{3pt}

\emph{DeepFool}~\cite{moosavi2016deepfool} find the nearest decision boundary to the data point $x$ in the $L_2$ norm and pushes the perturbation beyond this boundary. It iteratively minimizes the classifier $f$ around the input point until it produces a misclassification:
\begin{gather*}
    \tilde{x}_t = \tilde{x}_{t-1}+r_t\\
    \arg\min_{r_t} ||r_t||_2\\
    f(\tilde{x}_t,\theta)+\nabla_x f(\tilde{x}_t,\theta)^T r_t=0.
\end{gather*}

\vspace{3pt}

\emph{Carlini \& Wagner} (C\&W)~\cite{carliniwagner16towards} in the $L_\infty$ norm searches for the minimal adversarial perturbation producing a wrong classification. It solves the following optimization problem
\begin{gather*}
\min ||x-\tilde{x}||_\infty+c \cdot f(\tilde{x},t)\\
\tilde{x}\in[0,1]^m,
\end{gather*}
where $t$ is the target class.

\subsection{$\mathcal{R}_{v1}$ and $\mathcal{R}_{v2}$ are equivalent as $k\rightarrow \infty$.}
\label{sec:proof_regularizers}

Let  $R:\mathbb{R}^d\rightarrow\mathbb{R}^k$ be a random projection matrix, $R^\dagger:\mathbb{R}^k\rightarrow\mathbb{R}^d$ its pseudo-inverse and $x^\dagger:=R^\dagger Rx \in \mathbb{R}^d$ for $x\in\mathbb{R}^d$. 
For any given couple $(x,y)\in\mathbb{R}^d\times\mathbb{R}^K$ and projection matrix $R$ let us define 
\begin{align*}
\mathcal{T}_1(R) &:= \big|\big| \nabla_x \ell\big(f(x^\dagger,\theta),y\big)\big|\big|_2^2
\\
\mathcal{T}_2(R) &:= \big|\big| R \,\nabla_x \ell\big(f(x,\theta),y\big)\big|\big|_2^2.
\end{align*}

\begin{prop}
Let $R$ be a random projection matrix whose elements are sampled from $\mathcal{N}(0,1/k)$ and whose columns are orthogonal. Suppose that $x\in(\ker R)^\perp$  for all $x\in\mathbb{R}^d$. 
Then $$\mathbb{E}_R\big[\mathcal{T}_2(R)\big]= \frac{d}{k}\mathbb{E}_R\big[\mathcal{T}_1(R)\big].$$
\label{th:equivalence_perp_kernel}
\end{prop}

\begin{proof}
A random matrix $R$ and its pseudo-inverse $R^\dagger$ induce the direct sum decompositions
\begin{align*}
\mathbb{R}^d &= (\ker R)^\perp \oplus \ker R\\
\mathbb{R}^k &= \text{rank } R \oplus (\rank R)^\perp,
\end{align*}
where $\ker R$ is the kernel space of $R$ and $\rank R$ is the rank space of $R$.
Moreover, $R\big|_{(\ker R)^\perp}$ is an isomorphism with inverse $R^\dagger\big|_{\rank R}$ and $R^\dagger\big|_{(\rank R)^\perp}\equiv0$.
If $x\in(\ker R)^\perp$, then  $Rx\in\text{rank } R$ and $x^\dagger=R^\dagger Rx=x$, therefore 
$$
\mathcal{T}_1(R) = \big|\big| \nabla_x \ell\big(f(x,\theta),y\big)\big|\big|_2^2
$$
for all $x\in(\ker R)^\perp$.

Let $r_i\in\mathbb{R}^k$ be the orthogonal columns of $R$. Then $\mathbb{E}_R\big[||R||^2\big]=d/k$ and
\begin{align*}
\mathbb{E}_R\big[||R x||^2_2\big] 
&= \sum_{i,j=1}^d \mathbb{E}_R\big[r_i^T r_j\big] x_i x_j\\
& = \sum_{i=1}^d \mathbb{E}_R\big[r_i^T r_i\big] x_i^2 = \frac{d}{k} ||x||_2^2
\end{align*}
for any $x\in\mathbb{R}^d$. In particular
$\mathbb{E}_R\big[\mathcal{T}_2(R)\big]= \frac{d}{k} \mathcal{T}_1(R).$

\end{proof}

Notice that when $\mathbb{R}^k$ is high dimensional we can assume that the columns of any random matrix are orthogonal \cite{vershynin2018high}.

We now prove that Prop \ref{th:equivalence_perp_kernel}
holds for an arbitrary $x\in\mathbb{R}^d$.


\begin{prop}
Let $\pi_R:\mathbb{R}^d\rightarrow(\ker R)^\perp$ be an orthogonal projection and $k=\dim (\ker R)^\perp$. Then, for any $x\in\mathbb{R}^d$ and $\epsilon>0$
$$
P\big(||\pi_R(x)-x||_2^2 > ||x||_2^2 \,\epsilon\big)\leq \Big(1-\frac{\epsilon}{\pi}\Big)^k.
$$
\label{th:high_dim_ort_projection}
\end{prop}

\begin{proof}
Suppose that $\{v_1,\ldots,v_k\}\subset\mathbb{R}^d$ is a basis for $(\ker R)^\perp$, i.e. that $(\ker R)^\perp = \text{span}(v_1,\ldots,v_k)$. Then any $x\in\mathbb{R}^d$ can be decomposed as $x=\pi_R(x)+u$, where $\pi_R(x)\in (\ker R)^\perp$ and $u \in \ker R$. 

Let $\alpha_i$ be the angle between $v_i$ and $x$.
First, we observe that the projection $\pi_R(x)$ is smaller than any other projection on a single direction $v_i$
\begin{align*}
    ||\pi_R(x)-x||_2^2 
    &\leq \min_i ||\pi_{v_i}(x)-x||_2^2\\
    &= \min_i \big(||x||_2^2 |\sin \alpha_i|\big)\\
    & =||x||_2^2  \min_i |\sin \alpha_i|.
\end{align*}

For any choice of $\epsilon>0$
\begin{align*}
P\big(||\pi_R(x)-x||_2^2 >& ||x||_2^2 \epsilon\big) \\
&\leq P\big(||x||_2^2 \min_i |\sin \alpha_i| > ||x||_2^2 \, \epsilon\big)\\
& = P(\min_i |\sin \alpha_i| >\epsilon)\\
& = \prod_i P(|\sin \alpha_i| >\epsilon)\\
&\leq \prod_i P(|\alpha_i|>\epsilon)\\
&= \Big(1-\frac{\epsilon}{\pi}\Big)^k.
\end{align*}
\end{proof}

Notice that $1-\frac{\epsilon}{\pi}<1$, so $\big(1-\frac{\epsilon}{\pi}\big)^k\rightarrow 0$ as $k$ goes to $\infty$. Therefore, from  \ref{th:high_dim_ort_projection} we get $x^\dagger=\pi_R(x)\approx x$  and
$$
\nabla_x \ell\big(f(x^\dagger,\theta),y\big) \approx \nabla_x \ell\big(f(x,\theta),y\big)
$$
as $k\rightarrow \infty$.  This proves that Prop. \ref{th:equivalence_perp_kernel} is true for an arbitrary $x\in\mathbb{R}^d$ in the limit.

Assuming that $k=O(d)$ as $k\rightarrow \infty$, e.g. $\frac{d}{k}\rightarrow M>0$, the two regularization terms differ by a positive constant in the limit, i.e. they are equivalent if weighted w.r.t. $M$.

This proves the following theorem.


\begin{theorem}
Let $\mathcal{R}_{v1}$ and $\mathcal{R}_{v2}$ be the regularization terms defined in Eq.~\ref{eq:v1} and Eq.~\ref{eq:v2}, where $\mathcal{P}:\mathbb{R}^d\rightarrow\mathbb{R}^k$ is a random projection such that the elements of the orthogonal random matrix $R$ are sampled from $\mathcal{N}(0,1/k)$. If $k\in O(d)$ then $\mathcal{R}_{v1}\approx\mathcal{R}_{v2}$ as $k\rightarrow\infty$.
\end{theorem}

Notice that this punctual property on $\mathcal{R}_{v1}$ and $\mathcal{R}_{v2}$ also holds in expectation over the training data when $x$ is uniformly sampled from a compact subset of $\mathbb{R}^d$. Therefore, the equivalence between the two regularization terms holds in practice on mini-batches of training data.

\end{document}